\documentclass{article}

\usepackage[english]{babel}

\usepackage[margin=1.5in]{geometry}   

\usepackage{amsfonts}               
\usepackage{mathtools}              
\usepackage{amssymb}                
\usepackage{amsthm}                 
\usepackage[shortlabels]{enumitem}	

\usepackage{accents}                

\usepackage[bibstyle=authoryear, citestyle=authoryear, backend=biber, natbib]{biblatex}
\bibliography{references}

\newcommand{\ubar}[1]{\underaccent{\bar}{#1}}

\newtheorem{definition}{Definition}
\newtheorem{proposition}[definition]{Proposition}

\newtheorem{corollary}[definition]{Corollary}

\newcommand{\pred}{\mathsf{p}}
\newcommand{\Yhat}{\mathsf{\hat{Y}}}

\newcommand{\XX}{\mathcal{X}}
\newcommand{\DD}{\mathcal{D}}
\newcommand{\Acal}{\mathcal{A}}
\newcommand{\Ucal}{\mathcal{U}}

\newcommand{\ph}{p} 
\newcommand{\ip}{p^*}
\newcommand{\lp}{\ubar{p}}
\newcommand{\up}{\bar{p}}

\newcommand{\RR}{\mathbb{R}}
\newcommand{\EE}{\mathbb{E}}
\newcommand{\NN}{\mathbb{N}}

\title{A Unifying Perspective on Probabilities as Model Predictions}
\author{Benedikt Höltgen}
\date{\normalsize Hasso Plattner Institute, University of Potsdam}

\begin{document}

\maketitle

\begin{abstract}
Although probabilistic statements are ubiquitous, foundational disagreements persist about their understanding, as exemplified by debates between Bayesians and frequentists; moreover, it is unclear when and why acting on them actually leads to desirable outcomes. Here, we argue that every probability is the output of a \emph{prediction method}, that is, it depends on both a particular way of constructing abstractions and a way of transforming them into predictions. Through this, we provide a unifying perspective on supposedly different kinds of probabilities and show that even supposedly objective ones are model-dependent. We demonstrate that when a finite calibration criterion is met, one can anticipate the distribution of utilities for a given policy and inform successful decision-making on finite sets of events. Based on the notion of prediction methods, inductive arguments, and the probability calculus, we explain the feasibility of the calibration criterion in many settings. Overall, we develop a coherent perspective on probabilities and their use, connecting key intuitions behind other interpretations along the way.
\end{abstract}



\section{Introduction}
\label{s:intro}

\begin{quote}
    \raggedright
    Probabilities are as various as the faces to be seen at will in fretwork or paperhangings\\
    \raggedleft --- George \citet{eliot1871}: \textit{Middlemarch}
\end{quote}

We make probabilistic statements and use probabilistic reasoning all the time:
If the predicted `probability of rain' is sufficiently high, you bring your umbrella or even stay at home.
You decide to undergo surgery if this is thought to significantly `raise your chances' of recovery.
Given that such probabilistic statements permeate both science and our everyday lives, it is quite remarkable that it is still an open question what exactly we mean by them and how they are useful:
Do they refer to degrees of belief, to relative frequencies of repeated trials, or to physical properties?
The meaning of probability is considered to `bear at least indirectly, and sometimes directly, upon central scientific, social scientific, and philosophical concerns' \citep{hajekSEP}. 
In machine learning, the meaning of probability is increasingly recognised as a `pressing question' \citep{burhanpurkar2021};
as put by Cynthia Dwork, `without an answer to this definitional question, we don’t even know what it is that the ideal algorithm should satisfy' \citep{dwork2022}.  

A common view in both statistics and philosophy is that there are two kinds of probabilities, which one may refer to as aleatory and epistemic, respectively \citep[p. 13--15]{hacking1975}. 
Aleatory probabilities are grounded in the world, and potentially objective, for example in gambling or in other observable frequencies.
Epistemic probabilities are more speculative, linked to uncertain predictions and credences.\footnote{\citet{hacking1975} shows that this distinction is very old and can be found, e.g., in the distinction between \emph{chance} and \emph{probabilité} in the works of \citet{poisson1837} and \citet{cournot1843}.}
In philosophy, these concepts are often linked through the `Principal Principle', stating roughly that once a chance is learned, it should be adopted as credence.
Similarly, it is often assumed `that the job of statistics is to identify the data-generating process' \citep[160]{vovk2025}, or that in machine learning, `we would like to match the true data-generating distribution' \citep[130]{goodfellow2016}.

In contrast to these views, we develop a perspective that unifies these supposedly different kinds of probabilities, arguing that all probabilities are constructed and model-dependent.
In this work, we lay out a descriptive account of probability as predictions that are useful when they are finitely calibrated on certain sets.
We do not argue for a specific normative position about which predictions should be allowed to be called probabilities; however, our pragmatic perspective can shed new light on common notions of rational belief (abiding by the probability calculus) and decision-making (maximising expected utility).

The paper is structured as follows.
In Section~\ref{s:models}, we introduce the notion of prediction methods and show that they cover not only obvious examples like rain forecasts but also supposedly objective probabilities such as relative frequencies and gambling odds.
In Section~\ref{s:calibration}, we demonstrate how predictions that are finitely calibrated on relevant sets help us to make actually good decisions.
In Section~\ref{s:induction}, we elucidate why calibration is often feasible by drawing connections to the problem of induction and the probability calculus.
Lastly, we turn to to the literature on interpretations of probability and argue that our account satisfies general desiderata (Section~\ref{s:interpretation}) and captures key intuitions behind other interpretations (Section~\ref{s:comparison}).


\section{Behind Every Probability is a Prediction Method}
\label{s:models}

At the heart of our perspective on probability is the insight that each probability comes from a prediction method.
We first introduce the notion of prediction methods with an intuitive case and then walk through two perhaps less intuitive, seemingly aleatory examples.


\subsection{Predictors and prediction methods}
\label{ss:predictors}

All probability assignments involve predictions based on abstractions; the overall process of arriving at such a prediction we call a prediction method.
We distinguish the \textit{predictor}---the model that can be seen as a (mathematical) function---from the \textit{prediction method} that is applied to an actual situation.
The latter involves the construction of an abstraction (potentially including measurements) before applying the former (Figure~\ref{fig:pred_methods}). 
In the case of rain forecasts, the prediction method consists in first taking measurements (of temperatures, air pressure, etc.) and then feeding them to a computer model (the predictor) that outputs a prediction for the occurrence of rain.
%
\begin{figure}[h]
    \centering
    $ \left. \begin{array}{l} 
    \text{select predictor } \pred: \XX \times \Acal \to \RR \quad\\
    \ \\
    \text{construct abstraction } x \in \XX \quad
    \end{array} \right\} 
    \quad \text{compute } \pred(x,A)$
    \caption{A prediction method gives a prediction for an event $A$ in some situation by selecting a predictor $\pred: \XX \times \Acal \to \RR$ with $A \in \Acal$, constructing an abstraction $x \in \XX$ of the situation, and computing $\pred(x, A)$. For example, $\pred$ can be a computer model that predicts the event `rain' based on measurements $x$.}
    \label{fig:pred_methods}
\end{figure}

\begin{definition}[Predictor]\label{def:predictor}
    \ \\
    A \emph{predictor} is a function $\pred: \XX \times \Acal \to \RR$ on some set $\XX$ and algebra\footnote{An \textit{algebra} over a set $\Omega$ is a set of subsets of $\Omega$ that includes both $\Omega$ and the empty set and is closed under complements, finite unions, and finite intersections; this matters for Kolmogorov's axioms (Section~\ref{ss:axioms}).} $\Acal$.
\end{definition}
\begin{definition}[Prediction method]\label{def:pred_method}
    \ \\
    A \emph{prediction method} for an event $A \in \Acal$ is an implicit or explicit scheme for selecting a predictor $\pred: \XX \times \Acal \to \RR$ and constructing an abstraction $x \in \XX$ of the given situation.
\end{definition}
The predictor takes two arguments: the abstraction on which to base the prediction, and the event to predict.
In many settings, one of the two arguments is effectively ignored.
For rain forecasts, the algebra of events $\{\emptyset, \{rain\}, \{\neg rain\}, \{rain, \neg rain\}\}$ is only implicit.
By specifying a prediction $p_i$ for $\{rain\}$, we intuitively assign predictions $0, (1-p_i),$ and $1$ to the other events, respectively; we will turn to this in Section~\ref{ss:axioms}.
Note that the specification of events $A \in \Acal$ also involves choices of definition or measurement.
For example, how much rain counts as `no rain' or in which area it is recorded is more or less implicit---and depends on value judgements \citep{douglas2000}.
These choices can, however, be seen as external to the choice of prediction method (although the availability of methods can influence the choice of target), so we do not discuss them further.
Importantly, our notion of prediction does \emph{not} require that the predicted events lie in the future, it suffices that the observations/labels are not available to the predictor.

In the rain forecasting example, the abstraction made by the prediction method consists in taking specific measurements of temperature, air pressure, et cetera.
More generally, any sort of prediction requires focusing on a subset of all the information that could be taken into account---an abstraction of the situation.
Any given situation has an enormous amount of potentially relevant information; 
typically, we decide what to look at based on experience as well as common sense or expert knowledge.
For rain forecasts, temperature, and air pressure are more interesting quantities than the current GDP. 
Different models for rain prediction (the predictors) can also be based on different ways of measuring temperature---for example, different granularity, location, and timing of the measurements.
Different models can also work very differently---they may rely on simple look-up tables or sophisticated simulations.
They may even rely on human forecasters who also only use limited information for their forecasts.
While it is difficult to speak of human predictors as stable mathematical objects, they can arguably be approximated as such.

\subsection{Example: Symmetry-based predictions}
\label{ss:symmetry_based}

In many settings, probabilities are intuitively not thought to depend on modelling choices or a particular prediction method; this includes probabilities for gambling devices.
Assume you go to a casino where they offer a novel game based on a symmetrical 8-sided and a symmetrical 20-sided `die' (an octahedron and an icosahedron).
Given that it is an official casino, you assume that the dice are indeed symmetrical.
How do you make predictions?
You can represent any possible outcome that you wish to predict as the set of admissible combinations of faces, which can e.g. be represented as the algebra $\Acal = 2^{\{1,...,8\} \times \{1, ..., 20\}}$.
For the prediction, you presumably ignore the name of the croupier, the surface of the table, and so on, and only focus on the symmetry of the dice.
Each face of the octahedron corresponds to a prediction of $\frac{1}{8}$ and each face of the icosahedron corresponds to a prediction of $\frac{1}{20}$.
How exactly this reasoning is captured by a predictor $\pred: \XX \times \Acal \to \RR$ is under-determined; in particular, what counts as an argument versus as a part or parameter of the predictor:
You may consider a predictor that only takes symmetrical dice, such that the input space $\XX = \{x\}$ can be ignored.
Alternatively, you may take $\XX = \Delta_8 \times \Delta_{20}$ to be the set of all possible combinations of potentially biased 8-sided and 20-sided dice, of which you consider the element $x := (\frac{1}{8}, ..., \frac{1}{8}, \frac{1}{20}, ..., \frac{1}{20}) \in \XX$, reflecting your assumption of fairness.
You may also take a still larger $\XX$ that can also capture dice with other numbers of faces.
In any case, you proceed by transforming your abstraction of fair dice into a number in $[0,1]$ through a combinatorial model $\pred$.
You construct an abstraction and you calculate.

Such symmetry-based prediction methods in gambling situations typically assume that the device is `fair'.
Indeed, gambling devices are produced in such a way that each outcome should occur equally often, which allows us to make roughly accurate predictions about how often events occur on large samples.\footnote{\citet[4]{hacking1975} notes that already `[t]he dice in the cabinets of the Cairo Museum of Antiquity, which the guards kindly let me roll for a long afternoon, appear to be exquisitely well balanced.' It also appears that probabilistic calculations were already known to gamblers before mathematicians started to take an interest in it in the 17th century \citep{garber1979}.}
We know from experience that the process of rolling a fair die is so opaque and chaotic that it is practically impossible for us to predict better than uniformly.
If we were more proficient in discerning minuscule variations in die throws and background conditions (as Laplace's demon would be), we might be able to make finer predictions.
Indeed, Edward Thorp and Claude Shannon developed a device to better predict roulette outcomes which they successfully deployed in casinos in the 1960s \citep{thorp1998}.
After all, the assumption of a fair die or roulette wheel is a particular abstraction of (your beliefs about) the situation.

\subsection{Example: Frequency-based predictions}
\label{ss:frequency_based}

Predictions that are explicitly based on looking up relative frequencies of similar events use a very simple type of prediction method.
Such a method could be used for gambling instead of (or combined with) symmetry-based considerations.
There are also more interesting examples, such as medical risks.
Doctors typically base risk predictions on past experience, sometimes by explicitly looking up data about similar people.
In an example recently discussed in \citep{dawid2017}, which we shall return to later, Angelina Jolie got told that she had an $87\%$ risk of breast cancer---with the number presumably coming from statistical data about women with a particular genetic mutation.
Hence, the doctors used the following prediction method:
They chose a particular abstraction of Angelina, as a woman with this gene mutation, and then used a predictor that is basically a look-up table.
Presumably, women without that gene mutation get assigned into different categories (or `reference classes') for which there is enough data for doctors to believe that the relative frequency in this category is stable over time.
Different doctors may use different categories, that is, different abstractions.
As in the rain example, we can consider this as a case where the implicit 4-element algebra is ignored.
Alternatively, we could see it as an application of a more general predictor that can output predictions for different diseases, based on multiple look-up tables.
This would make $\Acal$ more complex by adding more diseases as fundamental events and means that $\XX$ needs to be fine enough that all relevant categories for all diseases can be distinguished.\footnote{One may also use a weaker set structure to not allow all intersections \citep{derr2023}.}


\section{Finite Calibration Makes Probabilities Useful}
\label{s:calibration}

The previous section argued that probabilities are outputs of prediction methods and thus constructed; this raises the question why we construct them and how they work.
Although it is often assumed that probabilistic predictions are useful for decision-making, it has not been demonstrated in general terms how or under which conditions this is the case.
In this section, we show how \textit{finite sets} of predictions are useful to us if they satisfy a form of calibration.
Although the technical details of the general perspective are almost trivial, it seems to not have been discussed before, let alone its relevance appreciated.
From this general understanding, the idea of expected utility maximisation and the more common understanding of calibration emerge as special cases.
For the purpose of this section, it is enough to think of predictions as arbitrary real numbers $p_i \in \RR$.\footnote{We will, however, demonstrate the benefits of satisfying Kolmogorov's axioms in later sections.}
As before, a prediction $p_i$ relates to an event $A_i$ and its label $y_i \in \{0,1\}$ where $y_i = 1$ or $y_i = 0$ denotes that the event does or does not occur, respectively.


\subsection{Predicting numbers of events}
\label{ss:result}

What is the difference between a prediction of $0.6$ and a prediction of $0.9$, given that the predicted event either does or does not occur?
An important difference surfaces when considering multiple predictions:
In general, of 100 events with prediction $0.6$, we intuitively expect roughly 60 to occur, whereas of 100 events with prediction $0.9$, we would expect roughly 90 to occur.
We can formalise this as a quality criterion for predictions called \textit{calibration}:
For a given set of events, the sum of our predictions should coincide with the number of occurring events.
\begin{definition}[Calibration]\label{def:calib}
    \ \\
    Predictions $p_1, ..., p_d \in \RR$ are said to be \emph{calibrated} for observations $y_1, ..., y_d \in \{0,1\}$ if they satisfy
    \begin{equation}
        \sum_{i=1}^d \ph_i = \sum_{i=1}^d y_i.
    \end{equation}
\end{definition}
%
Often, calibration is understood more narrowly as what \citet{dawid2017} calls `probability calibration', namely calibration \textit{on sets of equal prediction}.
While this unnecessarily narrow understanding of calibration may be partly due to historical reasons, as discussed in \citep{holtgen2023}, it is also particularly relevant in many settings (Section~\ref{ss:rain_example}).
Our definition is similar to the definition of calibration of \citet{dawid1985}, except that there, it is restricted to infinite sub-sequences of infinite sequences of outcomes whose averages are assumed to converge (which essentially presupposes the existence of objective, frequentist probabilities).
We argue that predicting how many events of certain sets will occur, i.e. calibration in the general sense, is the purpose for having probabilities in the first place.

The intuition for this is simple: 
If I knew that my predictions are calibrated on a set of events, then the sum of the predictions tells me how many of the set of events will occur.
This is very helpful, for example, in gambling settings: if I can reliably predict how often a repeatable event with given payoff will occur, then I know how much I should bet in each instance to come out positively in the end.
An important aspect here is that I do not care which of the events will occur, because the payout is always the same.
This is what calibration delivers: It tells you how many events will occur, without telling you which ones.

An implication of this is that calibration is only useful if I care equally about each event.
This notion of `caring equally' is captured in formal models by the condition that all events give me the same \emph{utility}.
Intuitively, a person's utility is a numerical representation of how much the person values the (non-)occurrence of an event (which is clearly an idealisation).
In our setting of binary events, we will use utility functions $u_i: \{0,1\} \mapsto \RR$ where $u_i(0)$ and $u_i(1)$ capture how much I value $y_i = 0$ and  $y_i = 1$, respectively.
Now, calibration can help us to foresee the (non-normalised) \emph{distribution of utilities} that I will receive:
If my predictions are calibrated on sets of equal utility, then I can predict the number of occurring events for each utility by summing up the relevant predictions.


\subsection{Predicting cumulative utility}
\label{ss:cumulative}

While we will come back to general utility distributions in Section~\ref{ss:calib_general}, we will for now focus on a particularly intuitive property of that distribution, which we call cumulative utility.

\begin{definition}[Cumulative utility]\label{def:utility}
    \ \\
    My \emph{cumulative utility} over a set of outcomes $\{y_1, ..., y_d\}$ given utility functions $u_1, ..., u_d$ is
    \begin{equation}
        \sum_{i=1}^d \left[ y_i u_i(1) + (1-y_i) u_i(0) \right].
    \end{equation}
\end{definition}
As $y_i$ denotes whether event $1$ or $0$ occurs, the cumulative utility is the \textit{sum over all utilities that I actually receive}.
This captures how well I will be off overall.
We now formally show that with suitably calibrated predictions, it is possible to estimate cumulative utility through a very familiar quantity.
The simple idea is that if for each utility value, I correctly predict how many events with this utility will occur, then I will correctly predict my cumulative utility.
Note that for our purposes, it would be mathematically equivalent to consider the average utility I get at each time step, i.e. the cumulative utility divided by $d$.
\begin{proposition}[Predicting cumulative utility]
    \label{prop:main_result}\ \\
    Let there be $d$ predictions $\ph_i \in \RR$ for binary outcomes $y_i \in \{0,1\}$, $i \in \{1,...,d\}$ with utility functions $u_i: \{0,1\} \to \RR$ and assume that the predictions are calibrated on sets of equal utility $u_i(0)$ and on sets of equal utility $u_i(1)$, formalised by the assumption that $\forall u' \in \Ucal:$
    \begin{equation}\label{eq:prop_assumption}
        \sum_{i: u_i(1) = u'} y_i = \sum_{i: u_i(1) = u'} \ph_i
        \quad\quad \text{and} \quad\quad
        \sum_{i: u_i(0) = u'} y_i = \sum_{i: u_i(0) = u'} \ph_i,
    \end{equation}
    where $\Ucal$ denotes the set of all values that the $u_i$ can take, i.e. $\Ucal := \bigcup_{1 \leq i \leq d} \{u_i(0), u_i(1)\} \subset \RR$.\\
    Then I can correctly predict my cumulative utility (LHS) via
    \begin{equation}\label{eq:main_result}
        \sum_{i=1}^d \left[ y_i u_i(1) + (1-y_i) u_i(0) \right] 
        = \sum_{i=1}^d \left[ \ph_i u_i(1) + (1-\ph_i) u_i(0) \right].
    \end{equation}
\end{proposition}
\begin{proof}\ \\
    \begin{align}
        \sum_{i=1}^d \left[ y_i u_i(1) + (1-y_i) u_i(0) \right]
        &= \sum_{u' \in \Ucal} \left( \sum_{i: u_i(1) = u'} y_i + \sum_{i: u_i(0) = u'} (1-y_i) \right) \cdot u' \label{eq:util_calib1} \\
        &= \sum_{u' \in \Ucal} \left( \sum_{i: u_i(1) = u'} \ph_i + \sum_{i: u_i(0) = u'} (1-\ph_i) \right) \cdot u' \label{eq:util_calib2} \\
        &=\ \sum_{i=1}^d \left[ \ph_i u_i(1) + (1-\ph_i) u_i(0) \right] \label{eq:util:calib3}
    \end{align}
    where for (\ref{eq:util_calib1}) = (\ref{eq:util_calib2}), we use the calibration criterion (\ref{eq:prop_assumption}).
\end{proof}
Given that the assumption of \textit{exact} calibration on all sets of \textit{equal} utility is very strong, we also show that approximate calibration (Appendix~\ref{app:approx_calib}) and calibration on sets of approximately equal utility (Appendix~\ref{app:approx_utility}) suffice for approximately correct predictions of the cumulative utility.
Furthermore, a weaker calibration criterion for imprecise predictions makes it possible to incorporate risk aversion (Appendix~\ref{app:imprecise}).
Note that the RHS of (\ref{eq:main_result}) is the sum over my expected utilities (my expected cumulative utility), in the conventional probabilistic framework where $\Yhat_i$ is the random variable with distribution $P_i$ that takes value $1$ rather than $0$ with probability $\ph_i$:
\begin{equation}\label{eq:expected}
    \sum_{i=1}^d \left[ \ph_i u_i(1) + (1-\ph_i) u_i(0) \right]
    = \sum_{i=1}^d \EE_{P_i}[u_i(\Yhat_i)].
\end{equation}
This means that the policy of expected utility maximisation (EUM) is the policy that actually maximises my cumulative utility if my predictions are calibrated!
To capture the idea of maximising utility, we need a notion of decisions between acts, which is not yet part of the setup. 
For simplicity, we consider $d$ binary decisions, at step $i$ consisting in a choice between $(p_i^a, u_i^a)$ and $(p_i^b, u_i^b)$.

\begin{corollary}[Comparing policies by expected utility]
    \label{cor:comparing}\ \\
    For $i \in \{1,...,d\}$, let there be predictions $\ph_i^a, \ph_i^b \in \RR$ for binary outcome $y_i \in \{0,1\}$ and utility functions $u_i^b, u_i^b: \{0,1\} \to \RR$.
    For $\pi \in \{a,b\}$, policy $\pi$ is then given by predictions $p_1^\pi, ..., p_d^\pi$ and utility functions $u_1^\pi, ..., u_d^\pi$.
    Assume that the predictions of both policies are calibrated on sets of equal utility in the sense of (\ref{eq:prop_assumption}) for their respective utility functions.\\
    Then, for $\Yhat_i^\pi$ and $P_i^\pi$ as in (\ref{eq:expected}), the policy with the higher expected cumulative utility $\sum_{i=1}^d \EE_{P_i^\pi}[u_i^\pi(\Yhat_i^\pi)]$ will actually provide the higher cumulative utility.
\end{corollary}
%
While EUM can be seen as a descriptive theory of human decision-making, it is often also assumed as a normative principle (e.g. by \citet{hedden2013}).
We can now give conditions under which this is a good policy in the sense that it leads to desired outcomes: when we are calibrated on sets of equal utility and when we care about maximising cumulative utility.
We now discuss a specific case of utility settings that has received much attention in statistics and machine learning, before we widen the scope again and consider cases where we are not interested in cumulative utility.


\subsection{Probability calibration}
\label{ss:rain_example}

We now illustrate the above with an example.
Let there be $d$ days where for $i \in \{1,...,d\}$, $\ph_i \in \{0, 0.1, 0.2, ..., 1\}$ is the daily rain forecast and $y_i \in \{0,1\}$ denotes whether it actually rains ($y_i=1$ denoting rain).
Now assume that across days, my attitude towards rain does not change but that it depends on whether I brought an umbrella:
Let my utilities be given by $u^a(1)=0$ and $u^a(0)=-1$ if I brought an umbrella and $u^b(1)=-3$ and $u^b(0)=0$ if I did not bring one.
Then my predicted utility when bringing an umbrella on day $i$ is 
\begin{equation}
    \ph_i \cdot u^a(1) + (1-\ph_i) \cdot u^a(0) = (1-\ph_i) \cdot (-1) = \ph_i - 1
\end{equation}
whereas my predicted utility when not bringing an umbrella is 
\begin{equation}
    \ph_i \cdot u^b(1) + (1-\ph_i) \cdot u^b(0) = -3 \ph_i.
\end{equation}
As $\ph_i - 1 > -3 \ph_i \Leftrightarrow \ph_i > 0.25$, I maximise predicted utility (per day) if I bring an umbrella on days where $\ph_i > 0.25$.
This policy induces the utility functions
\begin{equation}
    u_i = \begin{cases}
        u^a & \text{ if } \ph_i > 0.25 \\
        u^b & \text{ if } \ph_i < 0.25.
    \end{cases}
\end{equation}
The calibration criterion (\ref{eq:prop_assumption}) in Proposition~\ref{prop:main_result} for this case amounts to 
    \begin{equation} \label{eq:best_policy}
    \sum_{i: \ph_i < 0.25} y_i = \sum_{i: \ph_i < 0.25} \ph_i
    \quad\quad \text{and} \quad\quad
    \sum_{i: \ph_i > 0.25} y_i = \sum_{i: \ph_i > 0.25} \ph_i.
\end{equation}
If this condition is satisfied, my cumulative utility will coincide with the sum of my daily predicted utilities.
    
To assess the relative merits of this particular policy, we need to compare it with other policies.
(\ref{eq:best_policy}) is one instance of a prediction-dependent threshold policy, where I bring an umbrella whenever the predicted probability of rain is higher than a certain threshold (in this case, $0.25$).
For Proposition~\ref{prop:main_result} to apply to such a policy with any threshold $t \in [0,1]$, the calibration criterion is
    \begin{equation}
    \sum_{i: \ph_i < t} y_i = \sum_{i: \ph_i < t} \ph_i
    \quad\quad \text{and} \quad\quad
    \sum_{i: \ph_i > t} y_i = \sum_{i: \ph_i > t} \ph_i.
\end{equation}
Now note that since the possible utility functions are assumed to be the same each day, for this condition to be satisfied for all $t \in [0,1]$, it is enough to satisfy
\begin{equation}\label{eq:calib_prob}
    \forall v \in \{0, 0.1, 0.2, ..., 1\}:
    \sum_{i: \ph_i = v} y_i = \sum_{i: \ph_i = v} \ph_i.
\end{equation}  
Now this is just the common condition of calibration on sets of equal prediction, which is commonly expected of rain forecasters \citep{gigerenzer2005} and which `even inexperienced forecasters are capable of displaying', except for extreme predictions \citep[p. 191]{sanders1963}; see also \citet{murphy1977}.
Under this fairly benign assumption, choosing $0.25$ as my threshold maximises not only my \textit{predicted} utility but also my \textit{actual} cumulative utility among all threshold-based policies due to Proposition~\ref{prop:main_result}!
Hence, people can tailor their policies to their personal utilities and, thus, their decisions to the forecasts.
This also demonstrates why probabilistic forecasts are useful even in a deterministic world without `real' probabilities (cf. Section~\ref{ss:determinism}).
We would like to highlight that calibration on sets of equal prediction thus derives its importance (and prevalence) from their concurrence with the sets of equal utility when considering threshold-based policies.

This observation is closely linked to the connection between probability calibration and swap-regret that was shown for the first time by \citet{foster1998}.
In this work, a randomised forecasting algorithm which minimises the maximum regret under a permutation of predictions (the maximum swap regret) was used to achieve low calibration error; this was used to show that one can always achieve probability-calibrated forecasts, at least in probability.
While the importance of the reverse direction has since been appreciated \citep{noarov2024}, the benefit of calibration is commonly only framed in terms of regret; in our view, the connection to swap regret should be seen as a corollary of the more fundamental function of probability \emph{per se}, the accurate prediction of numbers of occurring events.
Closest to ours is the perspective of \citet{zhao2021} which focusses on the predictability of average loss, but is restricted to loss functions which depend only on prediction and outcome, not on the more general utility of the event.

Note that the calibration criterion was fairly benign in our rain example because the utilities are the same every day and we restricted our comparisons to the $10$ threshold-based policies (arguably the only sensible policies here).
The story would be more complex if we took the utility to also depend e.g. on wind speed (because it affects the efficacy of umbrellas)  or on the day of the week.
In general, for $d$ binary decisions, there are $2^d$ possible combinations of decisions, i.e. policies!
Accordingly, if we wanted to compare the cumulative utility of all possible choices via Proposition~\ref{prop:main_result}, this would lead to a very strong calibration criterion---in fact, it would require perfect binary predictions $p_i=y_i$.
This should not be surprising, as the best combination of decisions would be to always bring an umbrella if and only if it rains, which we can only ensure if we can discriminate perfectly between rainy and dry days.

It is, therefore, important to emphasise that calibration on sets of equal prediction should not be the only quality criterion for predictions.
Consider the rain example above: While the constant base rate predictor also satisfies the calibration criterion (\ref{eq:calib_prob}), more refined predictions would allow people to better tailor their decisions to their utilities.
Another property of interest is, thus, what is sometimes called sharpness or refinement, relating to the information content of a predictor \citep{degroot1983}.
The Brier score, to take a criterion important in rain forecasting, can be decomposed into two terms measuring probability calibration and sharpness, respectively \citep{sanders1963}.
These two properties are in tension in the sense that it is more difficult to be calibrated on more informative predictions.
While the focus of this work is how predictions can be useful in general, rather than their evaluation, we presently also mention some connections to the latter.


\subsection{Calibration revisited}
\label{ss:calib_general}

We now briefly consider alternatives to the cumulative utility as the quantity of interest.
Recall from Section~\ref{ss:result} that predictions calibrated on sets of equal utility not only allow us to predict the cumulative (or average) utility but the \textit{distribution} of utilities more generally.
Let $\DD := \{ \mu: \RR \to \NN_{\geq 0} \}$ denote the space of utility distributions, i.e. functions indicating how often different utility values occur.\footnote{For $u \in \RR$, $\mu(u) = n$ then means that the utility $u$ occurs $n$ times in the distribution described by $\mu$. Instead of such distributions, one may also think of multisets. Note that calibration on sets of equal utility generally only allows us to predict \textit{how often}, but not \textit{when} a specified utility will be received.\label{fn:utility_dist}}
Let $\Ucal_\mu := \{u \in \RR: \mu(u) > 0\}$ denote the support of a utility distribution $\mu \in \DD$, i.e. the set of utility values that do occur in the distribution $\mu$.
With this notation, we can describe the cumulative utility of a distribution $\mu \in \DD$ as $\sum_{u \in \Ucal_\mu} u \cdot \mu(u)$.
Another potentially relevant property of utility distributions is the smallest received utility, $\min \Ucal_\mu$.
For the umbrella policies, optimising for this property would mean that we should always bring an umbrella when $p_i > 0$, as we will otherwise incur a utility of $-3$ at some point---assuming that the calibration condition (\ref{eq:calib_prob}) holds.
The minimum is quite an extreme property of a distribution as it ignores most information about the distribution (in our case, all events except those with the lowest utility).
Optimising other properties of utility distributions will lead to other decision criteria but these questions are not the focus of this paper, interesting as they are.



In concurrent work, \citet{perdomo2025recht} claim that `[t]he utility of calibration comes in terms of communication' (p. 15) and
`emphasize that beyond [the] property of shared interpretation, calibration doesn’t mean much’ (p. 18).
Against this view, we highlight three interrelated perspectives on the importance of calibration.
First, the equivalence of swap regret and calibration highlights that the `best-response action' which maximises EU will fare better than the reverse policy. 
In the rain case, even just matching the base rate would already allow to always take the action compatible with the more probable event---which is not impressive but better than the opposite.
Arguably, this is only one part of a second, broader perspective on the choice of predictor given some policy and utility/loss:
Calibrated predictors guarantee better decision outcomes than miscalibrated predictors with the same discriminative power.
This perspective is explored further in many works connecting calibration to loss minimisation \citep{kleinberg2023, gopalan2023, feng2025, derr2025}.
The third perspective is closer to our basic idea of predicting utility distributions: If we assume calibration on relevant sets, we can choose among policies based on the utility distributions they will, respectively, lead to.
\citet{holtgen2025causal} demonstrate how this perspective can be fruitfully applied to causal inference settings.
These perspectives are different sides of the same gambling device; combining the latter two, one may suggest choosing policy and predictor together, based on the promised utility distribution and on how realistic it is that the required calibration conditions are met.
This highlights a fundamental question: How can we say anything about whether a set of predictions will be calibrated?


\section{How Is Calibration Possible?}
\label{s:induction}

Even if we showed calibration can make predictions useful, this only helps with understanding probability in the case that is also realistic to achieve calibration.
For arbitrary sets of predictions, there need not be a reason to assume that such a criterion would be satisfied.
This points to the importance of considering the methods that generated the predictions, which we discussed in Section~\ref{s:models}.
Based on the notion of calibration for predictions $p_1, ..., p_d$ as defined above, we can define calibration for predictors and prediction methods. 
For this, we consider the predictions of $d$ events $A_1,...,A_d \in \Acal$, with a prediction method that uses abstractions $x_1, ..., x_d \in \XX$.
\begin{definition}[Calibration of predictors and prediction methods]\label{def:calib_method}
    \ \\
    A predictor (or prediction method) is said to be calibrated on events $A_1,...,A_d \in \Acal$ for observations $y_1, ..., y_d$ if its predictions $p_1, ..., p_d$ are calibrated.
\end{definition}
%


Prediction methods provide a first way to draw connections between individual predictions, which is necessary to even start talking about satisfying criteria on \textit{sets} of events---but how can we hope for calibration on \textit{unseen} sets of events?
A first idea of providing calibrated predictors may be to directly optimise for it, in what has been called defensive forecasting \citep{vovk2005} or forecast hedging \citep{foster2021}.
This literature emerged in response to \citet{schervish1985} showing that calibration in the classical sense cannot be guaranteed by any algorithm.
\citet{foster1998} showed that, under very mild assumptions, probability calibration in the sense of a low ECE can be guaranteed (in probability) by a stochastic algorithm that minimises swap regret (for the Brier score).
However, such backward-looking algorithms often simply converge to (more or less stable) base rates, without refined predictions that allow well-informed policies.
While they can be adapted to smaller patches of the input space (see also the literature on multi-calibration \citep{hebert2018}), this is not much better than simply predicting the average per patch, thereby deciding in advance which patches get individual decisions.
In this, such approaches are more constrained than, e.g., human weather forecasters who were trained to first sort similar weather situations into `categories of likelihood of occurrence' and then predict a calibrated forecast per category \citep[200]{sanders1963}.
%






This example and those of Section~\ref{s:models} may suggest that induction about calibration always reduces to stable relative frequencies of repeated trials; this is also not the case.
Consider simulation models in the rain prediction example:
If the measurements are fine enough, it may be the case that no input $x_i \in \XX$ occurs more than once and no prediction is issued more than once, implying that there are no repeated trials.
Further examples are given by logistic regression or more complex Machine Learning models used for probabilistic predictions.
In general, calibrated predictors may rely on some structure in the relationship between inputs and labels that does not reduce to stable relative frequencies.
Otherwise, simulation-based and ML-based predictions could simply be replaced by `reference class forecasting' \citep{flyvbjerg2004}.
%

Empirically, we do see that prediction methods can often be designed such that they are (approximately) calibrated on sets of interest, which simply means that they neither systematically over- nor systematically under-predict.
Besides rain forecasts and the examples in Sections~\ref{ss:symmetry_based} and~\ref{ss:frequency_based}, we can point to machine learning (ML) models which often aim for calibration on sets of equal prediction.
It has been observed that especially modern, over-parameterised ML models need explicit post-processing, whereas others are automatically calibrated on sets of equal prediction \citep{guo2017}. 
One could argue that on a high level, humans also do something like this post-processing: if we are repeatedly over- or under-predicting (i.e. are not calibrated) on sets of interest, we will ideally notice that; since we do not know on which of the individual events our predictions were too low/high, we systematically increase/decrease our predictions in similar situations\footnote{Choosing a sensible notion of similarity here also requires experience. Nelson \citet[18]{goodman1972} already suspected `that rather than similarity providing any guidelines for inductive practice, inductive practice may provide the basis for some canons of similarity'.} in the future.
But why should future predictions then still be calibrated?

\subsection{Induction and the feasibility of calibration}
\label{ss:induction}


Any prediction method, indeed any prediction about the future, relies on an inductive assumption: that the future will resemble the past in some relevant way.
This relevance can be made more precise for our purpose: that a prediction method which has repeatedly proven to be (approximately) calibrated on some sets in the past will be (approximately) calibrated on similar sets in the future.
Below, we prove a formal result in support of this particular inductive assumption, similar to the argument for induction made by \citet{williams1947}.
In contrast to the cited work, we are dealing not only with integers but with real numbers, which is why we draw on an established concentration inequality.\footnote{Also note that we do not mean to provide an a priori justification for induction here: our goal is not to solve any (old or new) riddle of induction but to elucidate the role of probability in it.}
In particular, we make use of the combinatorial bound provided by Hoeffding's inequality for drawing without replacement.
\begin{proposition}[Calibration on samples from a population]\label{prop:induction}
    \ \\
    Take a predictor $\pred : \XX \times \Acal \to [0,1]$ and $N$ prediction instances represented by $(x_i, A_i) \in \XX \times \Acal$.
    Now consider drawing a `sample' of $d$ instances from the `population' of $N$ instances.
    Then the samples $\{j_1, ..., j_d\} \subset \{1,...,N\}$ whose average calibration error differs by more than $\epsilon$ from the average calibration error of the population, i.e. where
    \begin{equation}
        \frac{1}{d} \sum_{i=1}^d \left( p_{j_i} - y_{j_i} \right) 
        - \frac{1}{N} \sum_{i=1}^N \left( p_i - y_i \right)  
        \geq \epsilon,
    \end{equation}
    ($y_i$ denoting whether $A_i$ occurs and $p_i := \pred(x_i, A_i)$) make up for less than $\exp\left(-\frac{1}{2}d\epsilon^2\right)$ of all possible samples of that size.
\end{proposition}
\begin{proof}\ \\   
    Our result follows directly from Hoeffding's inequality for drawing without replacement.
    We simply insert $z_i := p_i - y_i$, $a=-1$ and $b=+1$ in the below statement taken from Proposition 1.2 of \citet{bardenet2015}:
    
    Let $Z = (z_1, ..., z_N)$ be a finite population of $N$ points and $Z_1, ..., Z_d$ be a random sample drawn without replacement from $Z$. Let
    \begin{equation}
        a := \min_{1 \leq i \leq N} z_i
        \quad \text{and} \quad
        b := \max_{1 \leq i \leq N} z_i.
    \end{equation}
    Then for all $\epsilon > 0$,
    \begin{equation}
        \mu \left[ \frac{1}{d} \sum_{i=1}^d Z_i - \frac{1}{N} \sum_{i=1}^N z_i \geq \epsilon \right]
        \leq \exp\left( - \frac{2d\epsilon^2}{(b-a)^2} \right),
    \end{equation}
    where $\mu$ measures the proportion of admissible combinations in drawing $d$ of the $N$ points.
\end{proof}
Hence, the average calibration error on large enough samples will mostly be close to the average calibration error of the whole population.
In a move analogous to that of \citet{williams1947}, we can also infer that if I am approximately calibrated on a large enough sample from a population or set of prediction instances, I will in most cases also be calibrated on the whole set and, thus, on similar sets in the future.
Let us illustrate the bound with concrete numbers.
If I have an average calibration error of $0.2$ on the whole population, then I will get a calibration error of less than $0.05$ in less than $10\%$ of possible samples of size $d=200$; for $d=500$, this ratio goes down to $0.4\%$. 
Hence, the vast majority of possible samples will not mislead me into thinking that I will be well-calibrated in the future in such a setting.
Note that the proposition only provides an upper bound, so that the actual number of non-representative samples will be lower still.
While this result does not prove the possibility of induction, it shows that calibration in the past is an indicator for calibration in the future---on sets that can be thought to be drawn from the same population.
Whether this is a sensible model in a given situation depends on whether there is reason to believe that the sample is unbiased.\footnote{An illuminating analysis of the difference between the means in terms of the bias of the sampling procedure is given by \citet{meng2018}.}

Another interesting implication of the result concerns the mixing of predictions from multiple, different calibrated prediction methods.
If $n$ prediction methods are calibrated on $d$ events each, then the resulting $n \cdot d$ predictions are clearly also calibrated on the $n \cdot d$ events;
Proposition~\ref{prop:induction} can also be applied to this larger set of predictions, now inferring from the population to subsets:
It shows that most large enough subsets of these $n \cdot d$ predictions will also be approximately calibrated, even though they come from a mix of prediction methods.
This is important because it shows that the Proposition is not only relevant for predictions from the same prediction method.
In sum, we can give arguments why, in certain cases, we expect to be calibrated on future events---but we can never be sure:
`Nature will always maintain her rights, and prevail in the end over any abstract reasoning whatsoever’ \citep[5.1.2]{humeEnquiry}.


\subsection{Extrapolating calibration}
\label{ss:axioms}

Another way of generating sets of calibrated predictions is through certain other sets of calibrated prediction---by using probabilistic reasoning.
Even attentive readers probably missed the interesting fact in Proposition~\ref{prop:main_result} that $1-p_i$ automatically emerged as the prediction for $1-y_i$, without imposing Kolmogorov's axioms.
We now show more generally that predictors need to satisfy these axioms (in their second argument) in order to be calibrated on certain sets.\footnote{
Related observations for the case of sets of equal prediction have been made by \citet{vanfraassen1983}.}
Take a predictor $\pred : \XX \times \Acal \to \RR$ and $d$ prediction instances represented by $(x_i, A_i) \in \XX \times \Acal$ with $p_i := \pred(x_i, A_i)$ and let $y_i \in \{0,1\}$ denote whether $A_i$ occurs.
%
Let $y_A, y_B, y_{A \cup B}, y_\Omega \in \{0,1\}$ denote whether events $A, B, A \cup B, \Omega \in \Acal$ occur at the last instance $d$.
Let $\Acal$ be an algebra over some set $\Omega$ where $\Omega$ is a sure event:
It exhausts all possibilities; that is, for its label, it is known that $y_\Omega=1$.
\begin{enumerate}
    \item \emph{Non-negativity:}
    If there is a nontrivial subset $I := \{i : p_i < 0\} \subset \{1,...,d\}$ where $\pred$ predicts negative values, then $\pred$ cannot be calibrated on this subset, regardless of whether the predicted events occur:
    \begin{equation}
        \sum_{i \in I} p_i < 0 \leq \sum_{i \in I} y_i.
    \end{equation}

    \item \emph{Normalisation:}
    Let $A_d = \Omega$ and $\pred$ be calibrated on the set $\{1,...,d-1\}$.
    Then $\pred$ is calibrated on $\{1,...,d\}$ if and only if $\pred(x_d, \Omega) = 1$, regardless of $x_d$.

    \item \emph{Additivity:}
    Let $A, B \in \Acal$ be disjoint events in the sense that $y_A + y_B \leq 1$ (i.e. it cannot be that both labels are equal to $1$ at the same instance); this implies $y_A + y_B = y_{A \cup B}$.
    Now assume $\pred$ is calibrated on the set $S := \{(x_1, A_1),...,(x_{d-1}, A_{d-1}), (x_d, A), (x_d, B)\}$, where $\pred$ is used for two predictions at instance $d$. Then
    \begin{align}
        \pred(x_d, A \cup B) + \sum_{i=1}^{d-1} p_i
        - y_{A \cup B} - \sum_{C \in J} y_i
        =\ &\pred(x_d, A \cup B) + \sum_{i=1}^{d-1} p_i
        - y_{A} - y_{B} - \sum_{i=1}^{d-1} y_i \\
        =\ &\pred(x_d, A \cup B)- \pred(x_d, A) - \pred(x_d, B) \\
        &+ \pred(x_d, A) + \pred(x_d, B) + \sum_{i=1}^{d-1} p_i
        - y_{A} - y_{B} - \sum_{i=1}^{d-1} y_i \nonumber \\
        =\ &\pred(x_d, A \cup B) - \pred(x_d, A) - \pred(x_d, B) \label{eq:add_last}
    \end{align}
    where the last step uses the assumption of calibration on $S$.
    So under that assumption, $\pred$ is calibrated on $\{1, ..., d\}$ with $A_d= A \cup B$ if and only if $\pred(x_d, A \cup B) = \pred(x_d, A) + \pred(x_d, B)$, regardless of $x_d$.
\end{enumerate}
The sets that allow if-and-only-if statements are quite specific here; in this sense, it resembles Dutch book arguments, where any single inconsistency can in theory be exploited indefinitely.
Here, however, the implications are more practical: If someone is perfectly calibrated on forecasting `rain' but does not obey the probability axioms on one `no rain' forecast, then for some utility functions (in the setting of Section~\ref{ss:rain_example}), the best-response policy is guaranteed to lead to sub-optimal decisions due to miscalibration.

We can also motivate the definition of conditional probabilities by the demand for calibration (somewhat analogous to definitions via relative frequencies).
Consider the task of predicting events $A, B \in \Acal$ at $d$ instances.
For ease of presentation, assume that all $d$ inputs coincide, i.e. $x_1 = ..., x_d = x \in \XX$. 
This allows us to drop $\pred$'s dependence on $x \in \XX$ and consider a predictor $\pred: \Acal \to [0,1]$ in the following derivation; a more general version is presented in Appendix~\ref{app:axioms}.
Now assume $\pred$ to be calibrated on $A \cap B$ and on $B$ across the $d$ instances, where $y_i^{A \cap B}$ and $y_i^B$ denote whether $A \cap B$ and $B$ occur at instance $i \in \{1,...,d\}$, respectively.
That is, assume $\sum_{i = 1}^d \pred({A \cap B}) = \sum_{i=1}^d y_i^{A \cap B}$ and $\pred(B) = \frac{1}{d} \sum_{i=1}^d y_i^B > 0$.
Then $\pred$ is calibrated on $A | B$ for $\{i: y_i^B=1\}$ (i.e. for the set of steps where $B$ occurs, see first line below) if and only if it satisfies $\pred(A | B) = \frac{\pred(A \cap B)}{\pred(B)}$:
\begin{align*}
    \sum_{i : y_i^B = 1} \pred(A | B) &= \sum_{i : y_i^B = 1} y_i^A\\
    \sum_{i : y_i^B = 1} \pred(A | B) &= \sum_{i : y_i^B = 1} y_i^{A \cap B} 
    && \text{(since $y_i^A = y_i^{A \cap B}$ when $y_i^B = 1$)}\\
    \sum_{i : y_i^B = 1} \pred(A | B) &= \sum_{i = 1}^d y_i^{A \cap B} 
    && \text{(since $y_i^{A \cap B}=0$ when $y_i^B = 0$)}\\
    \sum_{i : y_i^B = 1} \pred(A | B) &= \sum_{i = 1}^d \pred({A \cap B}) 
    && \text{(by calibration of $\pred$ on $A \cap B$)}\\
    \pred(B) \cdot \sum_{i = 1}^d \pred(A | B) &= \sum_{i = 1}^d \pred({A \cap B}) 
    && \text{(by calibration of $\pred$ on $B$)}\\
    \pred(A | B) &= \frac{\pred({A \cap B})}{\pred(B)}.
\end{align*}
%
Summing up, the probability calculus can be seen as a sound and complete system for generating calibrated predictions on certain sets from calibrated predictions on related sets.
While this does not settle the question whether all predictors need to follow the probability calculus (i.e. that they are probability measures in their second argument), it does provide a \emph{pro tanto} reason.


\section{A Unified Interpretation of Probability}
\label{s:interpretation}

\begin{quote}
    Norms of belief are as remote from empirical claims about nature as is Hume's simpler subjectivism. Propensity theories of probability propose a physical property that cannot be recorded and does not necessitate or preclude any occurrence. [\dots] any limiting-frequency claim is consistent with any claim about any finite collection of events.

    \raggedleft
    --- Clark \citet{glymour2001} 
\end{quote}
Russell's famous dictum that `probability is the most important concept in modern science, especially as nobody has the slightest notion what it means' (cited by \citet[p. 582]{bell1945}) is almost a century old; but while there have certainly been many new developments, a satisfying interpretation is still lacking.
The purpose of this section is to argue that his lacuna can be filled by the perspective on probabilities put forward in the present paper.
To make this argument, we now specifically relate our account to the literature on interpretations of probability.
In the most authoritative up-to-date treatment of the subject, \citet{hajekSEP} asks `what do we want from our interpretations \textit{of probability}, specifically?’ (original emphasis) and then answers by suggesting a list of desiderata (drawing on  \citet{salmon1966}).
Some of these we have already covered above:
Our account satisfies `non-triviality' (not just zero and one) and ‘admissibility with respect to this or that axiomatization’ (motivating the axioms of the probability calculus); it also illuminates `ampliative inferences' in the sense that it allows to reason about the justification of probabilistic statements based on other probabilistic statements (Section~\ref{s:induction}).
We now, in turn, discuss Hájek's remaining desiderata: the applicability to science, rational belief, frequencies, and rational decision making.







\subsection{Relation to science and the question of (in-)determinism}
\label{ss:determinism}

\begin{quote}
    During the nineteenth century it became possible to see that the world might be regular and yet not subject to universal laws of nature. A space was cleared for chance.

    \raggedleft
    --- Ian \citet{hacking1990} 
\end{quote}

Assessing the applicability to science means checking the compatibility with scientific practice and current scientific theories.
In particular, the question of whether the universe is deterministic is sometimes taken to bear directly on how we should think about probability.
For example, David \citet[120]{lewis1980} thought that objective or physical probabilities rely on indeterminism.
Karl \citet{popper1959} also proposed the propensity account of probability (according to which probabilities are physical properties) in the context of Quantum Mechanics (QM).
We should briefly note that QM does not require the world to be indeterministic, given that different, empirically indistinguishable interpretations of QM disagree on this question (not even getting into the question of scientific realism).
So there is certainly no need to presuppose this.
But even if QM came with some notion of true probabilities, it is unclear that they would be of any relevance to the probabilities we deal with day-to-day, for two reasons: 
First, QM probabilities need to be described by a more general theory of probability than that axiomatised by Kolmogorov \citep{streater2000}.
Second and more importantly, even for a coin flip, we would never have access to the true QM-based probabilities:
We are neither able to determine the initial conditions, that is, the complete wave function, nor to take into account the extremely high number of occurring quantum interactions.\footnote{This would require a QM version of Laplace's demon.}
The resulting values may, thus, vastly differ from any predictions we are able to make---which, as we have shown, are still useful.
In sum, we do not see compelling reasons to suppose either determinism or indeterminism, nor to think that indeterminism at the level of QM would contribute much to probabilistic reasoning.
It is therefore a strength of our account that, showing how probabilities can be constructed and used, it remains agnostic regarding the question of determinism.

Indeed, the intuition that probabilities are objective may depend less on QM and more on the often strong interpersonal agreements about `correct' prediction methods e.g. for gambling.
In the words of Michael \citet[31]{strevens2006}, probabilities in such settings `have attained a certain kind of stability under the impact of additional information. This stability gives them the appearance of objectivity, hence of reality, hence of physicality'.
We argued in Section~\ref{ss:symmetry_based} that this appearance is misleading, as illustrated by the roulette story of Thorp and Shannon.
In line with this, the `erosion of determinism' indeed did not follow the advent of QM but of higher-level statistical regularities discovered during the previous century, as captured in Hacking's epigraph above.
It is also important to note that the higher-level sciences depend heavily on abstractions such as the ones that feature in our notion of prediction methods.
Abstractions have been observed to be both part of scientists' tacit knowledge \citep{polanyi1958}, especially the `ability to recognize a given situation as like some and unlike others' \citep[195]{kuhn1969}, and a substantial part of conscious scientific work \citep{danks2015, potochnik2017}---yet they, arguably, still remain an under-explored topic.


\subsection{Relation to (rational) degrees of belief}
\label{ss:relation_degrees}

\begin{quote}
    Chances are degrees of belief [\dots]; not those of any actual person, but in a simplified system to which those of actual people, especially the speaker, in part approximate.

    \raggedleft
    --- Frank P. \citet{ramsey1928} 
\end{quote}
While probabilities are often said to have a direct connection to degrees of belief, we argue that the notion of probability does not depend on degrees of belief: prediction methods can be used in a purely mechanical way to get desirable outcomes on aggregate, without an entity involved that is commonly held to have beliefs.
A simple machine that takes measurements and uses a predictor could make decisions based on probabilities without it being plausible to ascribe to it beliefs that come in degrees.
However, probabilities can also be used to \textit{describe degrees of belief} under uncertainty.
In most cases, it is difficult to pin down a particular prediction method, especially as human predictions tend to be qualitative.
But humans also take only specific information into account and exploit regularities such as symmetries or stable relative frequencies.
In some cases, the gap between human reasoning and quantitative prediction methods can become fairly small---it appears that some people, modestly described as `super-forecasters', are particularly good at making calibrated quantitative predictions \citep{mellers2015}.
The notion of prediction methods can also shed light on imprecise notions of (subjective) uncertainty.
Particularly vague degrees of belief or disagreements between different methods can be represented by imprecise predictions (Appendix~\ref{app:imprecise})\footnote{Such interval probabilities were also considered by \citet{definetti1962} as models of degrees of belief in `the case of a number of decision-makers who have to make a collective decision, and, second, the case of a single individual who experiences a `kind of personality dissociation'' \citep[348]{feduzi2012}.} while Knightian uncertainty \citep{knight1921} corresponds to the absence of a trusted prediction method.
In this sense, Section~\ref{s:models} also tells us an idealised story of human reasoning:
Consciously or not, humans often implement something close to prediction methods---in that sense, probabilities can model human degrees of belief, a view also expressed in Ramsey's epigraph.

While probabilities should not be taken as actual degrees of belief, we can also \textit{model consistent decision-making} by humans or machines as if they had certain degrees of belief:
\citet{savage1972} famously showed that actions which follow certain consistency criteria can be viewed as maximising expected utility for implicit utility functions and probabilities.
Expected utility maximisation (EUM) is, thus, often taken to be an approximation of human decision-making, that is, as a descriptive theory:
We tend to make decisions such that good outcomes seem more likely to us.
There are, of course, considerable caveats. 
The most crucial ones are arguably diminishing marginal utility and risk aversion, already highlighted by \citet[172]{ramsey1926} and analysed e.g. in \citep{wakker1994}.
In Ramsey's words, EUM as a modelling tool is an `artificial system of psychology, which like Newtonian mechanics can, I think, still be profitably used even though it is known to be false' \citep[173]{ramsey1926}.
So, we can connect degrees of belief to probabilities, by modelling reasoning and decision-making through prediction methods and Savage-style decision theory--but stop short of equating them.

%
There is, of course, some flexibility in deciding what to use the word `probability' for.
One may want to use it for the assignments in Savage-style models, that is, for implicit degrees of belief that can be assigned whenever some agent acts consistently.
It seems, however, more consistent with everyday use to reserve it for the predictions themselves, for weather predictions and coin flips, and to say that we can model consistent decision as if they follow EUM under certain probabilities.
As shown in Section~\ref{ss:axioms}, we can get calibrated predictions from calibrated predictions of related events using the probability calculus.
This provides a \textit{pro tanto} reason for considering the probability axioms to constitute constraints on rationality.
We emphasise again that we do not put forward a normative theory of rational belief or action here, but a descriptive perspective on probability that can illuminate its perceived connections to rationality.


\subsection{Relation to frequencies and the reference class problem}
\label{ss:relation_frequencies}

\begin{quote}
    If we are asked to find the probability holding for an individual future event, we must first incorporate the case in a suitable reference class. An individual thing or event may be incorporated in many reference classes, from which different probabilities will result. This ambiguity has been called the \textit{problem of the reference class}.
    
    \raggedleft
    --- Hans \citet{reichenbach1949} 
\end{quote}
Although our perspective implies that probabilities are constructed, it also explains their strong connection to observed relative frequencies.
Indeed, if we restricted calibration to sets of equal probability (as calibration is sometimes understood), the relationship would be even closer: 
Then, the calibration condition would be equivalent to the definition of probability in finitary frequentism:
\begin{equation}\label{eq:finite_freq}
    \sum_{i=1}^n p_i = \sum_{i=1}^n y_i
    \quad\quad \Leftrightarrow \quad\quad
    p_i = \frac{1}{n} \sum_{i=1}^n y_i.
\end{equation}
%
Instead of taking this as a definition, we think it more adequate to see it as a special case of our main quality criterion.
Not just because such finitary definitions are problematic \citep{hajek1996},\footnote{\citet{glymour2001} argues for what he calls an instrumentalist and approximate version of finite frequentism which takes probabilities to be descriptions of frequencies rather than defining the former through the latter. While this is not too far from our somewhat pragmatic approach in spirit, his focus is more on the description of populations through distributions and he rejects the relevance of decision theory.\label{fn:glymour}} but also because it would imply too narrow an evaluation criterion.

The ties between frequentism and our account become particularly clear in reference to the so-called reference class problem.
Its metaphysical version is a problem for objectivist theories like frequentism that claim a unique true probability for each event \citep{hajek2007}.
The epistemic version concerns the question of how a reference class should be chosen for a given event, considering that different choices would lead to different probabilities.
This has led Hájek, for example, to argue that conditional probabilities are actually primitive, as probabilities are always conditional on a certain conceptualisation of events.
While our predictors do resemble them, they are not strictly speaking conditional probabilities, as $\XX$ is just an arbitrary set without well-defined probabilities; as we show in Section~\ref{ss:axioms}, it makes more sense to consider conditional probabilities to further depend on a particular (perhaps implicit) abstraction $x \in \XX$ (see the implicit joint `conditioning' in Appendix~\ref{app:axioms}).
\citet{hajek2007} also supposed that, rather than a marginal probability, `[v]arious frequentists could tell us the conditional probability that John Smith will live to age 61, \textit{given} that he is a consumptive Englishman aged 50' (ibid., p. 582, original emphasis).
However, there can be different mortality tables resulting in different ratios---more generally, the choice of abstraction does not yet fix the prediction.
This aspect is clear for the prediction method perspective, as different methods may use the same scheme of abstraction but different predictors, relying e.g. on different mortality tables.\footnote{In Machine Learning, the phenomenon that prediction methods can give different predictions despite using the same abstractions \textit{and} using the same data \textit{and} achieving the same average loss is known under the name of `predictive/model multiplicity' (see e.g. \citep{breiman2001,black2022}).}
For these reasons, we agree with \citet[p. 23]{freedman1997} that `probability is a subtler idea than relative frequency'.
We would also argue that the reference class problem is not actually a problem.
Different prediction methods may be calibrated on different sets, so one can choose a prediction method that promises calibration on sets of interest.
This relates to to the more general idea of the `goal-dependence in scientific ontology'  \citep{danks2015}.


\subsection{Relation to rational decision-making and individual predictions}
\label{ss:individual_preds}

\begin{quote}
    Legends of prediction are common throughout the whole Household of Man. God speaks, spirits speak, computers speak. Oracular ambiguity or statistical probability provides loopholes, and discrepancies are expunged by Faith.

    \raggedleft
    --- Ursula K. \citet{leguin1969}: \textit{The Left Hand of Darkness}
\end{quote}
Our discussions have focused on sets of predictions rather than individual ones.
This is not a coincidence, as we take probabilities to not be free-floating numbers but to rely on prediction methods which are useful when sets of predictions are calibrated.\footnote{Of course, singletons can be (approximately) calibrated when predictions are (approximately) 0 or 1.}
But what exactly is the relation between a prediction and the corresponding event?
And can we evaluate the quality of a single non-trivial prediction?
That is, is a prediction of 0.6 better than a prediction of 0.4 if the predicted event occurs?
What should we do if we only get a single prediction (for some level of utility)?

The first three questions all relate to the dependence of a prediction on the method that generated it.
It is important to emphasise that the probability is not a property of the event, as it is constructed and depends on the choices of both the abstraction and the predictor.
As discussed in Sections~\ref{ss:symmetry_based} and~\ref{ss:determinism}, gambling setups only appear to have objective probabilities because of their relative stability under additional information
We also mentioned the example of Angelina Jolie, who stated `My doctors estimated that I had an 87 per cent risk of breast cancer', with the number presumably coming from statistical data about women with a particular genetic mutation.
\citet{dawid2017} asks, `Was Angelina (or her doctors) right to interpret it as her own individual risk?' (p. 3456).
On our account, they were---with the qualification that this risk is model-dependent and constructed rather than objective and discovered---as is any other probability.
Which prediction method is most useful depends on which sets we want to be calibrated on (although the perfect binary predictor is always optimal).
In hindsight, we can usually say which prediction would have been good or correct.
The validation of prediction methods cannot, however, be thus reduced to comparisons between individual predictions, not even with proper scoring rules \citep{gneiting2007}.
They do allow us to put a number on our intuition that 0.6 is somehow a better prediction than 0.4 if the predicted event occurs; but so does any notion of calibration error (as the $\ell_1$ loss deployed in Appendix~\ref{app:approx_calib}).
After all, proper scoring rules are meant to be `appropriate for evaluating and comparing forecasters who \emph{repeatedly present their predictions}' \citep[p. 12, emphasis added]{degroot1983}.

The fourth and last question about acting on single predictions is related but more complex.
In general, we suggest that policies rather than single actions should be the subject of justification and evaluation.
An ex-post evaluation of a decision would ignore the prediction and just consider whether an alternative decision would have been better in hindsight---this is not particularly helpful.
Instead, what is familiar also from legal and ethical reasoning (especially deontological, but even rule-consequentialist), is to judge decisions by the reasons or maxims that they were based on.\footnote{This has been stated in particularly succinct form by Maurice \citet[p. 9]{merleau1955}: `Il n'y a pas des d\'ecisions justes, il n'y a qu'une politique juste.'}
For example, we showed that maximising expected utility is a good policy if we can assume calibration on sets of equal utility and wish to maximise cumulative utility (Section~\ref{ss:result}).
As noted before, being calibrated for all possible combinations of decisions would require perfect discrimination.
In the umbrella example of Section~\ref{ss:rain_example}, we showed that the calibration criterion can be more benign when comparing a more restricted set of sensible policies.
But the problem is more difficult e.g. when we only have a few predictions for particularly grave events:
If we only make a few high-stakes decisions, such as a choice of treatment for breast cancer (where one may even argue that the concept of numerical utility breaks down), it seems too big of an assumption to hope for calibration on such a small set.
That being said, the combinatorial reasoning from Section~\ref{ss:induction} also extends to single predictions: Good calibration in the past gives \textit{some} reason to believe in low calibration error on the single prediction, i.e. the more reason to believe in the event, the higher the prediction:
The strength of this mathematically-grounded pro-tanto reason is monotonous in the number of events, so some (even if small) reason to believe will remain.\footnote{I am grateful to Gunnar König for pushing me on this point.}
In general, for such situations, it may be more sensible to be risk-averse than in low-stakes settings where there are multiple events with comparable utility (cf. \cite{buchak2013, thoma2019}).\footnote{This creates an asymmetry for the doctor-patient relationship, but also for algorithmic predictions, similar to the insurance setting \citep{frohlich2024}. 
C.S. \citet{peirce1878}, in contrast, thought that when probabilistic reasoning is confronted with limited trials, `logicality inexorably requires that our interests [\dots] must not stop at our own fate, but must embrace the whole community'.}
A way to model this would be via imprecise calibration as explored in Appendix~\ref{app:imprecise}.


\section{Comparison with Conventional Interpretations}
\label{s:comparison}

\citet{hajekSEP} notes that `[e]ach interpretation that we have canvassed seems to capture some crucial insight into a concept of [probability], yet falls short of doing complete justice to this concept.' 
Any new satisfactory account of probability should, thus, be expected to make proponents of other accounts feel vindicated on some aspects that are particularly close to their hearts.
We think that this is the case for our notion of probabilities as outputs of prediction methods aiming to predict numbers of occurring events.
It is interesting to note, for example, that our predictors $\pred: \XX \times \Acal \to \RR$ resemble the confirmation function central to logical accounts of probability, such as that of \citet{carnap1950} or \citet{keynes1921} (with precursors as early as Leibniz, cf. \citep{hacking1975}).
Our predictors, however, are neither objective relations nor relations between propositions---they are functions of abstractions in a set $\XX$ and events in an algebra $\Acal$.
In this section, we briefly survey a number of other prominent interpretations and highlight what we take to be the most interesting similarities and differences w.r.t. our account.

\textit{Bayesianism} roughly posits that probability and its theory are concerned with degrees of belief and rationality constraints thereon.
What we agree with is that probabilities are constructed and that it is misguided to search for true probabilities.
However, we ground them in prediction methods rather than degrees of belief (Section~\ref{ss:relation_degrees}) and highlight that these methods aim to track structure in sets of observations.
This makes it possible to replace notions of internal cohesion or rationality with that of empirical calibration, and thereby a guide to decision-making that guarantees good outcomes.
A Bayesian account that is particularly close to ours is that of Philip \citet{dawid2017}.\footnote{Also the work of \citep{shafer2019}, which builds on Dawid's earlier work, is quite close to ours in spirit; but their focus is on testing probability forecasts, rather than how they are useful.}
On the one hand, his suggestion to arrive at `probability forecast[s] by assessing the odds at which I would be willing to bet' (p. 3471) is clearly Bayesian in the tradition of \citep{definetti1937}.
On the other hand, he also suggests to evaluate individual predictions on aggregate data via calibration---although its precise scope and relevance do not become entirely clear.
In particular, it remains unclear why calibration on future data is important and on which (finite/infinite) sets it matters.\footnote{For example, his notion of $H$-based calibration seems to require calibration on all sets that cannot be further distinguished---which can amount to calibration on individual datapoints.}
In comparison, our notion of prediction methods focuses on (potentially) inter-subjective models and the role of abstraction, which is decoupled from the events $A \in \Acal$ that we wish to be calibrated on.
In a way, then, we posit a variant of Bayesianism without degrees of belief or betting and with a more concrete connection to the world, enabling not only the avoidance of sure loss in Dutch books but successful action in everyday life.

Hypothetical \textit{frequentism} can be defined as the suggestion that `the probability of an attribute A in a reference class B is the value the limiting relative frequency of occurrences of A within B would be if B were infinite' \citep{hajekSEP}.
This captures the intuition of identifying probabilities with ratios in repeated trials.
While this sounds very different to our account at first glance, we already discussed two similarities in Section~\ref{ss:relation_frequencies}:
One is the dependence of individual probabilities on other events and on a choice of abstraction (via prediction methods, in our case), leading us to a generalisation of the reference class problem. 
Furthermore, equating probabilities with relative frequencies is a special case of our notion of calibration, which we consider for finite sets.
We do reject the jump to declaring that probabilities themselves are `out there' in any interesting sense.
The finite frequentist account of \citet{glymour2001}, mentioned in footnote~\ref{fn:glymour}, provides, in a sense, an intermediate account.

Karl Popper abandoned frequentism in favour of his \textit{propensity} account because the former could not make sense of sequences with few trials.
He thus proposed that frequentists should alter their theory by letting it `say that admissible sequences must be either virtual or actual sequences which are \textit{characterised by a set of generating conditions}---by a set of conditions whose repeated realisation produces the elements of the sequence' \citep[p. 34, original emphasis]{popper1959}.
This is still an objectivist theory but dispenses with the reliance on infinite trials, instead invoking a new sort of mysterious property (especially in the case of a deterministic universe, which Popper did not seem to assume).
We argued that relevant probabilities are independent of `true' probabilities that may or may not be implied by Quantum Mechanics (Section~\ref{ss:determinism}).
Propensity accounts often have a frequentist flavour, indirectly highlighting the importance of sets of events.
It is interesting to note that, as the equivalence classes of generative conditions are idealisations (ignoring background conditions, cf. Section~\ref{ss:symmetry_based}), they can be seen as abstractions made by prediction methods.
However, propensities are typically thought to be physical rather than model-dependent, which is in stark contrast to our account---although the relevant literature sometimes also invites a reading of model-dependent propensities.

Another interesting interpretation of probability is the \textit{best-systems account} of David \citet{lewis1994}, which also posits objective chances:
On this view, `the chances are what the probabilistic laws of the best system say they are' (p. 480).
`The best system is the one that strikes as good a balance as truth will allow between simplicity and strength. [\dots] If nature is kind, the best system will be robustly best [\dots] It's a reasonable hope'  (ibid., p. 478f).
Now this account presupposes what may seem a tremendous kindness of nature as well as a perhaps weak notion of truth and objectivity---the latter fits well into Lewis' Humean view on laws of nature.
What is interesting here about Lewis' account is that it resonates with the hope for a best level of predictive depth expressed by \citet[3465]{dawid2017}---in turn similar to the `primary resolution' of \citet{li2021}.
Indeed, Dawid could be seen as linking the best-systems view on probability with our more pragmatic notion of model-based predictions.
The clearest differences on the side of Lewis are the integration within a more global systematisation of the universe and the belief in objectivity, hinging on the existence of a privileged description.
If there were an objectively best predictor and we assigned to it some notion of truth, these differences would blur.\footnote{This is perhaps not surprising given the subjectivism and pragmatism of Frank Ramsey, whom Lewis credits with a first formulation of a best-systems approach.}
However, this hope for or pretension of objectivity is also what Clark Glymour criticises in typical frequentist takes.

In line with our analysis, Glymour thinks that central problems with Bayesianism and frequentism lie, respectively, in the neglect of empirical claims and the unnecessary stipulation of objectively true probabilistic statements:
\begin{quote}
    The sometimes bitter debates between those who describe themselves as frequentists and those who describe themselves as subjective Bayesians has often turned on charges by the former that the latter abandon the ``objectivity'' of science and by the latter that the former dissemble about the ``subjectivity'' of their probability judgements. 
    My belief is that, among statisticians anyway, the dispute often confuses content with justification. 
    The ``objectivity'' of the frequentists is in the content of their probability judgements, which, while usually stated as about an unempirical probability, are often really vague empirical claims about finite frequencies. That sort of objectivity is genuinely lost in subjective Bayesian interpretations. 
    The ``subjectivity'' kept hidden by frequentists is that there is often no explicit justification beyond their own opinion for aspects of their empirical claims. 
    That subjectivity can be made entirely explicit without sacrificing the objective--that is empirical--content of frequency claims, and its recognition does not require, or even invite, recourse to subjective probability. 
    Bayesian criticisms do address a confused and uncertain frequentist statistical practice, in which the point of making empirical claims is often forgotten or fudged. \citep[p. 299f]{glymour2001}
\end{quote}
We have argued that our account avoids these problems by stating that probabilities are constructed rather than discovered while still taking their justification directly from empirical observations.
Even more, we connect successful decision-making with empirical evaluation and assumptions about induction through a general notion of calibration, which has not been considered a central concept by any of the conventional accounts.


\section{Conclusion}
\label{s:discussion}

Relying on the notions of finite calibration and prediction methods, we have provided a more or less pragmatic account of probabilities and how they are useful for decision-making.
We showed that if predictions satisfy an (often feasible) calibration criterion, then it is possible to predict the distribution of utilities that a given policy will yield.
In particular, the sum of one's predicted utilities will match the actual cumulative utility, which can provide a rationale for expected utility maximisation.
A central element of our account is the semi-formal notion of prediction methods that construct abstractions of given situations and feed them to a model.
Arguably, the novelty here consists less in the consideration of predictions than in the connections drawn to abstractions and to successful decision-making via calibration in very general terms.
Indeed, we argue that this perspective elucidates the relationship between gambling odds and rain predictions, uniting the alleged two faces of probability by tying together abstraction, forecasting, probability theory, and empirically successful decision-making.



The understanding of probability also has important implications for machine learning.
For example, it underscores the relevance of evaluating calibration beyond sets of equal predictions, as already explored by \citet{dawid2017, holtgen2023}.
Furthermore, it underscores that one should be wary of the often-invoked concept of a `true distribution' from which one can `sample' once one steps outside of the casino or other highly controlled settings.
Given the centrality of probability for causality, many considerations also spill over to the latter; indeed, we take a deeper dive into causal inference from this perspective in concurrent work, also highlighting the role of calibration \citep{holtgen2025causal}.
It has been observed that algorithmic predictions based on machine learning tend to convey an air of authority and objectivity, as the many choices involved in data collection (abstraction scheme) and model tuning (choice of predictor) often remain beneath the surface \citep{moss2022}.
This is particularly relevant for the justification of predictions that inform decisions about people \citep{holtgen2025gend}.
Our work highlights that probabilities, e.g. of finding a job, are not properties of people; instead, they depend on the selected abstraction and model, which, in turn, depends on data about other people.
Hence also our answer to Cynthia Dwork's question from the introduction: There is no ideal algorithm, as there are no true probabilities to uncover, and different algorithms can be better suited for different goals.
While we hope that this work helps to sharpen the view on probability, a sea of open questions still calls for further exploration.


\subsection*{Acknowledgments}
For helpful feedback on previous versions, I would like to thank Ben Jantzen,
Bob Williamson, Elisa Nguyen, Jannik Thümmel, Kate Vredenburgh, Konstantin
Genin, Rabanus Derr, and Timo Freiesleben.
This work was funded by the German Federal Ministry of Education and
Research (BMBF): Tübingen AI Center, FKZ: 01IS18039A. I also thank the International
Max Planck Research School for Intelligent Systems (IMPRS-IS) for
their support.

\appendix


\section{Generalising Proposition~\ref{prop:main_result}}
\label{app:generalising}

\subsection{Approximate calibration}
\label{app:approx_calib}

Here, we generalise Proposition~\ref{prop:main_result} to only require approximate calibration---we give a bound on how large the calibration error on each set of equal utility can be in order to keep the difference between predicted and cumulative utility below some $\epsilon > 0$.

\begin{proposition}[Predicting cumulative utility: Approximate calibration]\label{prop:approx_cal}
    \ \\
    We assume the same setting as in Proposition~\ref{prop:main_result} except that we now require all utilities to be positive---one may otherwise simply shift the values to a positive domain.
    If we then replace condition (\ref{eq:prop_assumption}) with the assumption that $\forall u' \in \Ucal$,
    \begin{align}
        \left| \sum_{i: u_i(0) = u'} y_i - \sum_{i: u_i(0) = u'} \ph_i \right|
        \leq \frac{\epsilon}{2 \cdot u' \cdot |\Ucal|}
    \end{align}
    and the same for $u_i(1)$, then
    \begin{align}
        \left| \sum_{i=1}^d \left[ y_i u_i(1) + (1-y_i) u_i(0) \right] 
        - \sum_{i=1}^d \left[ \ph_i u_i(1) + (1-\ph_i) u_i(0) \right] \right| \leq \epsilon.
    \end{align}
\end{proposition}
\begin{proof}\ \\
    \begin{align}
        & \left| \sum_{i=1}^d \left[ y_i u_i(1) + (1-y_i) u_i(0) \right] 
        - \sum_{i=1}^d \left[ \ph_i u_i(1) + (1-\ph_i) u_i(0) \right] \right| \nonumber \\
        &= \left| \sum_{u' \in \Ucal} \left( 
        \left( \sum_{i: u_i(1) = u'} y_i - \sum_{i: u_i(1) = u'} \ph_i \right) 
        + \left( \sum_{i: u_i(0) = u'} \ph_i - \sum_{i: u_i(0) = u'} y_i \right) 
        \right) \cdot u' \right|\\
        &\leq \sum_{u' \in \Ucal} \left( 
        \left| \sum_{i: u_i(1) = u'} y_i - \sum_{i: u_i(1) = u'} \ph_i \right| 
        + \left| \sum_{i: u_i(0) = u'} \ph_i - \sum_{i: u_i(0) = u'} y_i \right| 
        \right) \cdot u'\\
        &\leq \sum_{u' \in \Ucal} \left(
        \frac{\epsilon}{2 \cdot u' \cdot |\Ucal|} + \frac{\epsilon}{2 \cdot u' \cdot |\Ucal|}
        \right) \cdot u' \\
        &= \epsilon
    \end{align}
\end{proof}
While we use a symmetric $\ell^1$ loss here, it may be interesting to also look into other measures of error.
For example, for settings where under-prediction and over-prediction are valued differently, it may be instructive to look into asymmetric error functions.


\subsection{Approximate utility level sets}
\label{app:approx_utility}

Here, we generalise Proposition~\ref{prop:main_result} to only require calibration on sets of approximately equal utility:
For this, we divide the utility spectrum into bins of some size $\delta > 0$ and bound the resulting difference between predicted and cumulative utility by a term dependent on $\delta$ and the number of predictions $d$.

\begin{proposition}[Predicting cumulative utility: Approximate utility]\ \\
    We assume the same setting as in Proposition~\ref{prop:main_result} except that we now require all utilities to be positive---otherwise, one may simply shift the values to a positive domain.
    We partition the interval of relevant utilities from the lowest $u_i(a)$ to the highest $u_i(a)$ with $i \in \{1,...,d\}, a \in \{0,1\}$
    into bins $B_1, ..., B_m$ of size $\leq \delta$.
    If we then replace condition (\ref{eq:prop_assumption}) with the assumption that $\forall k \in \{1,...,m\}$,
    \begin{align}
         \sum_{i: u_i(0) \in B_k} y_i = \sum_{i: u_i(0) \in B_k} \ph_i 
         \quad\quad\text{and}\quad\quad
         \sum_{i: u_i(1) \in B_k} y_i = \sum_{i: u_i(1) \in B_k} \ph_i
    \end{align}
    then
    \begin{align}
        \left| \sum_{i=1}^d \left[ y_i u_i(1) + (1-y_i) u_i(0) \right] 
        - \sum_{i=1}^d \left[ \ph_i u_i(1) + (1-\ph_i) u_i(0) \right] \right| \leq \delta \cdot d.
    \end{align}
\end{proposition}
\begin{proof}\ \\
    The maximal mismatch occurs when for each bin $B_k$ and each $a \in \{0,1\}$, one half of the $\{i: u_i(a) \in B_k\}$, we have $(y_i - \ph_i)=1$ and $u_i(a) = m_k+\delta/2$ whereas for the other half, $(y_i - \ph_i)=-1$ and $u_i(a) = m_k-\delta/2$, with $m_k$ denoting the midpoint of $B_k$.
    This gives
    \begin{align}
    \forall k \in \{1,...,m\}, a \in \{0,1\}: \quad
         \left| \sum_{i: u_i(a) \in B_k} (y_i - \ph_i)  \cdot u_i(a) \right|
	\leq \left| \sum_{i: u_i(a) \in B_k} \delta/2 \right|
	= b_k^a \cdot \delta /2
    \end{align}
    where $b_k^a := |\{1 \leq i \leq d\ |\ u_i(a) \in B_k\}|$. 
    Therefore,
    \begin{align}
        &\left| \sum_{i=1}^d \left[ y_i u_i(1) + (1-y_i) u_i(0) \right] 
        - \sum_{i=1}^d \left[ \ph_i u_i(1) + (1-\ph_i) u_i(0) \right] \right| \\
        &\leq  \left| \sum_{i=1}^d  [y_i \cdot u_i(1) - \ph_i  \cdot u_i(1)] \right|
        +  \left| \sum_{i=1}^d  [(1-y_i) \cdot u_i(0) - (1-\ph_i) \cdot u_i(0)] \right| \\
        &=  \left| \sum_{i=1}^d [(y_i - \ph_i)  \cdot u_i(1)] \right|
        + \left| \sum_{i=1}^d   [(y_i -\ph_i)  \cdot u_i(0)] \right| \\
        &\leq \sum_{k=1}^m \left( 
        \left| \sum_{i: u_i(1) \in B_k} (y_i - \ph_i)  \cdot u_i(1) \right| 
        + \left| \sum_{i: u_i(0) \in B_k} (y_i -\ph_i)  \cdot u_i(0) \right| 
        \right) \\
        &\leq \sum_{k=1}^m \left( 
         b_k^1 \cdot \delta /2
        +  b_k^0 \cdot \delta /2
        \right) \\
        &= \delta \cdot d
    \end{align}
\end{proof}


\subsection{Imprecise calibration}
\label{app:imprecise}

We now consider imprecise forecasts which give interval predictions $[a,b] \subset \RR$ and represent them as tuples $\ip = (\lp,\up) \in \RR^2$ of the lower and upper probability.
This allows for a weaker calibration criterion where the number of occurring events need not exactly match the sum of predictions, but should lie between the sum of the lower and the sum of the higher predictions.
\begin{definition}[Imprecise calibration]\label{def:i_calib}
    \ \\
    Imprecise predictions $\ip_1, ..., \ip_d$ are said to be \emph{imprecisely calibrated} for observations $y_1, ..., y_d \in \{0,1\}$ if they satisfy
    \begin{equation}
        \sum_{i=1}^d \lp_i \leq \sum_{i=1}^d y_i
        \quad \text{and} \quad
        \sum_{i=1}^d \up_i \geq \sum_{i=1}^d y_i.
    \end{equation}
\end{definition}
%
Note that for our definition, the vacuous forecast that always predicts $(0,1)$ is always imprecisely calibrated.\footnote{Similar issues are discussed in the literature on imprecise probability \citep{walley1991}, particularly for Brier-style scoring rules \citep{seidenfeld2012} and randomness \citep[Prop. 9]{deCooman2022}.}
One could also apply the criterion of imprecise calibration to a set of precise predictions, by simply converting every precise prediction $p_i$ into an imprecise forecast $[p_i - \epsilon, p_i + \epsilon]$ for some $\epsilon$---this epsilon may also monotonically decrease in $d$ to account for lower variance on larger sets.
\begin{proposition}[Predicting cumulative utility, imprecise version]
    \label{prop:ip_result}\ \\
    Let there be $d$ imprecise predictions $\ip_i \in \RR^2$ for binary outcomes $y_i \in \{0,1\}$, $i \in \{1,...,d\}$ with utility functions $u_i: \{0,1\} \to \RR$ and assume that the predictions are imprecisely calibrated on sets of equal utility $u_i(0)$ and on sets of equal utility $u_i(1)$ (formalised in (\ref{eq:ip_prop_assumption}) below).\\
    Then I can correctly predict a range for my cumulative utility (LHS) via
    \begin{equation}
        \sum_{i=1}^d \left[ y_i u_i(1) + (1-y_i) u_i(0) \right] 
        > \sum_{i=1}^d \left[ \lp_i u_i(1) + (1-\lp_i) u_i(0) \right]
    \end{equation}
    and
    \begin{equation}
        \sum_{i=1}^d \left[ y_i u_i(1) + (1-y_i) u_i(0) \right] 
        < \sum_{i=1}^d \left[ \up_i u_i(1) + (1-\up_i) u_i(0) \right]
    \end{equation}
\end{proposition}
    The proof is analogous to that of Proposition~\ref{prop:main_result}, now with the calibration assumptions
    \begin{align}
    \begin{split}\label{eq:ip_prop_assumption}
        \sum_{i: u_i(1) = u'} y_i > \sum_{i: u_i(1) = u'} \lp_i
        \quad\quad &\text{and} \quad\quad
        \sum_{i: u_i(0) = u'} y_i > \sum_{i: u_i(0) = u'} \lp_i,\\
        \sum_{i: u_i(1) = u'} y_i < \sum_{i: u_i(1) = u'} \up_i
        \quad\quad &\text{and} \quad\quad
        \sum_{i: u_i(0) = u'} y_i < \sum_{i: u_i(0) = u'} \up_i.
    \end{split}
    \end{align}
This allows people to not only optimise their utility but to also take risk-averse or risk-seeking inclinations into account---selecting policies not based on the expected exact cumulative utility but on e.g. the lowest or highest estimation of it. 
Here, we can see an analogy between the move from deterministic to probabilistic and the move from precise to imprecise predictions:
The former allows people to take their (cardinal) preferences into account (Section~\ref{ss:rain_example}), whereas the latter allows them to take their risk aversion into account.
If we know that we will be calibrated, risk aversion does not make much sense.
Cases where we are less sure of it can be represented by an assumption of imprecise calibration.
Note that this notion of risk-aversion also captures unwillingness to bet, for decisions between the utility function of a bet and the constant zero utility function with $u(0)=u(1)=0$.


\section{Conditional probabilities, generalised}
\label{app:axioms}

We here generalise the analysis of \textit{conditional probabilities} in Section~\ref{ss:axioms}.
Consider a predictor $\pred: \XX \times \Acal \to [0,1]$, events $A_1, ..., A_d, B \in \Acal$, and inputs $x_1, ..., x_d \in \XX$.
We assume $\pred(x_1, B)=...=\pred(x_d, B)$ and that $\pred$ is calibrated on $\{(x_i, A_i \cap B) : 1 \leq i \leq d\}$ and $\{(x_i, B) : 1 \leq i \leq d\}$,
that is, 
\begin{equation}\label{eq:cond_cal1}
    \sum_{i = 1}^d \pred(x_i, {A_i \cap B}) = \sum_{i=1}^d y_i^{A_i \cap B} 
\end{equation}
and
\begin{equation}\label{eq:cond_cal2}
    \pred(x_1, B) = \frac{1}{d} \sum_{i=1}^d y_i^B > 0.
\end{equation}
%
%
Then $\pred$ is calibrated on $\{(x_i, A_i | B) : y_i^B=1\}$ (i.e. for the set of steps where $B$ occurs) if and only if it satisfies 
\begin{equation}\label{eq:cond_character}
    \sum_{i = 1}^d \pred(x_i, A_i | B) = \sum_{i = 1}^d \frac{\pred(x_i, A_i \cap B)}{\pred(x_i, B)},
\end{equation}
as we derive below.
In particular, a sufficient condition is 
\begin{equation}
    \pred(x_i, A_i | B) = \frac{\pred(x_i, A_i \cap B)}{\pred(x_i, B)}.
\end{equation}
%
Now consider the special case where $A_i = ... = A_d =: A$ and  $x_i = ... = x_d =: x$.
Here, the familiar definition of conditional probabilities 
\begin{equation}
    \pred(x, A | B) = \frac{\pred(x, A \cap B)}{\pred(x, B)}
\end{equation}
is necessary and sufficient for $p$ to be calibrated for predictions of $A | B$ based on inputs $x$ on the set of steps where $B$ occurs.
That is, making predictions for $A | B$ rather than $A$ allows us to be calibrated on the set where $B$ occurs (which predictions for $A$ would usually not be).

Now the promised derivation of the characterisation (\ref{eq:cond_character}):
\begin{align*}
    \sum_{i : y_i^B = 1} \pred(x_i, A_i | B) &= \sum_{i : y_i^B = 1} y_i^{A_i | B}\\
    \sum_{i : y_i^B = 1} \pred(x_i, A_i | B) &= \sum_{i : y_i^B = 1} y_i^{A_i \cap B} 
    && \text{(since $y_i^{A_i | B} = y_i^{A_i} = y_i^{A_i \cap B}$ when $y_i^B = 1$)}\\
    \sum_{i : y_i^B = 1} \pred(x_i, A_i | B) &= \sum_{i = 1}^d y_i^{A_i \cap B} 
    && \text{(since $y_i^{A_i \cap B}=0$ when $y_i^B = 0$)}\\
    \sum_{i : y_i^B = 1} \pred(x_i, A_i | B) &= \sum_{i = 1}^d \pred(x_i, {A_i \cap B}) 
    && \text{(by (\ref{eq:cond_cal1}))}\\
    \pred(x_1, B) \cdot \sum_{i = 1}^d \cdot \pred(x_i, A_i | B) &= \sum_{i = 1}^d \pred(x_i, {A_i \cap B}) 
    && \text{(by (\ref{eq:cond_cal2}))}\\
    \sum_{i = 1}^d \pred(x_i, A_i | B) &= \sum_{i = 1}^d \frac{\pred(x_i, {A_i \cap B})}{\pred(x_i, B)}.
\end{align*}

\printbibliography

\end{document}